\documentclass[english,final]{llncs}

\usepackage[utf8]{inputenc}
\usepackage[T1]{fontenc}
\usepackage[english]{babel}
\usepackage{graphicx}

\usepackage{textcomp}

\usepackage{setspace}

\usepackage{xcolor}

\usepackage{hyperref}
\hypersetup{
  a4paper,
  bookmarksnumbered,
  bookmarksopen,
  bookmarksopenlevel=0,
  colorlinks=false,
    citebordercolor=teal,
    linkbordercolor=pink,
	pdfauthor={},
	pdftitle={Tractable combinations of global constraints},
	pdfsubject={Constraint Programming},
	pdfkeywords={CS,constraints,databases,tractability},
   urlcolor=black,
   implicit=false,
}

\usepackage{verbatim}

\usepackage{complexity}

\usepackage{mathtools}
\usepackage{amsfonts}

\usepackage{cite}

\title{Tractable Combinations of Global Constraints}

\author{
David~A. Cohen\inst{1}
\and
Peter~G. Jeavons\inst{2}
\and \\
Evgenij Thorstensen\thanks{Supported by EPSRC grant EP/G055114/1}\inst{2}
\and
Stanislav \v{Z}ivn\'y\thanks{Supported by a Senior Research Fellowship from
Warwick's DIMAP.}\inst{3}
}

\institute{Department of Computer Science, Royal Holloway, University of London, UK\\\email{d.cohen@rhul.ac.uk}
\and
Department of Computer Science, University of Oxford, UK\\ 
\email{firstname.lastname@cs.ox.ac.uk}
\and
Department of Computer Science, University of Warwick, UK\\
\email{s.zivny@warwick.ac.uk}}

\newcommand{\tw}{\ensuremath{\mathsf{tw}}}
\newcommand{\CSP}{\ensuremath{\mathsf{CSP}}}

\newcommand{\Core}{\ensuremath{\mathsf{Core}}}

\newcommand{\eqt}{\ensuremath{\mathsf{equiv}}}
\newcommand{\vars}{\ensuremath{\mathsf{vars}}}
\newcommand{\hyp}{\ensuremath{\mathsf{hyp}}}
\newcommand{\TDD}{\ensuremath{\mathsf{twDD}}}
\newcommand{\join}{\ensuremath{\mathsf{join}}}
\newcommand{\iv}{\ensuremath{\mathsf{iv}}}
\newcommand{\ext}{\ensuremath{\mathsf{ext}}}

\newcommand{\languageword}{catalogue}
\newcommand{\Languageword}{Catalogue}
\newcommand{\languagesymbol}{\ensuremath{\mathcal C}}

\DeclarePairedDelimiter{\tup}{\langle}{\rangle}

\newcommand{\omitted}[1]{}

\begin{document}

\maketitle

\begin{abstract}
  We study the complexity of constraint satisfaction problems 
  involving global constraints, i.e.,~special-purpose constraints provided
  by a solver and represented implicitly by a parametrised
  algorithm. Such constraints are widely used; indeed,
  they are one of the key reasons for the success of constraint
  programming in solving real-world problems.

  Previous work has focused on the development of efficient
  propagators for individual constraints. In this paper, we identify a
  new tractable class of constraint problems involving global
  constraints of unbounded arity.  To do so, we combine structural
  restrictions with the observation that some important types of
  global constraint do not distinguish between large classes of
  equivalent solutions.
\end{abstract}

\section{Introduction}

Constraint programming (CP) is widely used to solve a variety of
practical problems such as planning and scheduling
\cite{vanHoeve2006:global-const,Wallace96practicalapplications}, and
industrial configuration \cite{pup-cpaior2011,bin-repack-datacenters}.
The theoretical properties of constraint problems, in particular the
computational complexity of different types of problem, have been
extensively studied and quite a lot is known about what restrictions
on the general \emph{constraint satisfaction problem} are sufficient
to make it tractable
\cite{struct-decomp-stacs-et,bulatov05:classifying-constraints,Cohen08:unifiedstructural,Gottlob00:acomparison,grohe07-hom-csp-complexity,stoc-Marx10-submod-width}.

However, much of this theoretical work has focused on problems where
each constraint is represented \emph{explicitly}, by a table of
allowed assignments.

In practice, however, a lot of the success of CP is due to 
the use of special-purpose constraint types for which the software tools 
provide dedicated algorithms~\cite{Rossi06:handbook,Gent06:minion,wallace97-eclipse}. 
Such constraints are known as \emph{global constraints} 
and are usually represented \emph{implicitly} by an algorithm in the solver. 
This algorithm may take as a parameter a
\emph{description} that specifies exactly which kinds of assignments a
particular instance of this constraint should allow.

Theoretical work on global constraints has to a large extent focused
on developing efficient algorithms to achieve various kinds of 
local \emph{consistency} for individual constraints.
This is generally done by 
pruning from the domains of
variables those values that cannot lead to a satisfying assignment
\cite{Bess07:reasoning-global-const,Samer11:constraints-tractable}.
Another strand of research has explored when it is possible to 
replace global constraints by collections of explicitly
represented constraints \cite{bess-nvalue}. These techniques allow
faster implementations of algorithms for \emph{individual constraints}, 
but do not shed much light on the complexity of 
problems with multiple \emph{overlapping} global constraints, 
which is something that practical problems frequently require.

As an example, consider the following family of constraint problems involving clauses and
cardinality constraints of unbounded arity.

\begin{example}
  \label{example:Running}
  Consider a family of constraint problems on a set of Boolean variables
  $\{x_1,x_2,\ldots,x_{3n}\}$ (where $n = 2,3,4,\ldots$), with the following five constraints:
  \begin{itemize}
  \item $C_1$ is the binary clause $x_1 \vee x_{2n+1}$;
  \item $C_2$ is a cardinality constraint on $\{x_1,x_2,\ldots,x_n\}$ 
  specifying that exactly one of these variables takes the value 1;
  \item $C_3$ is a cardinality constraint on $\{x_{2n+1},x_{2n+2},\ldots,x_{3n}\}$
   specifying that exactly one of these variables takes the value 1;
  \item $C_4$ is a cardinality constraint on $\{x_2,x_3,\ldots,x_{3n}\}-\{x_{2n+1}\}$
   specifying that exactly $n+1$ of these variables takes the value 1;
  \item $C_5$ is the clause $\neg x_{n+1} \vee \neg x_{n+2} \vee \cdots \vee \neg x_{2n}$.
  \end{itemize}

This problem is illustrated in Figure~\ref{fig:RunningExampleG}.

\begin{figure}[ht]
  \centering
  \includegraphics[scale=0.35]{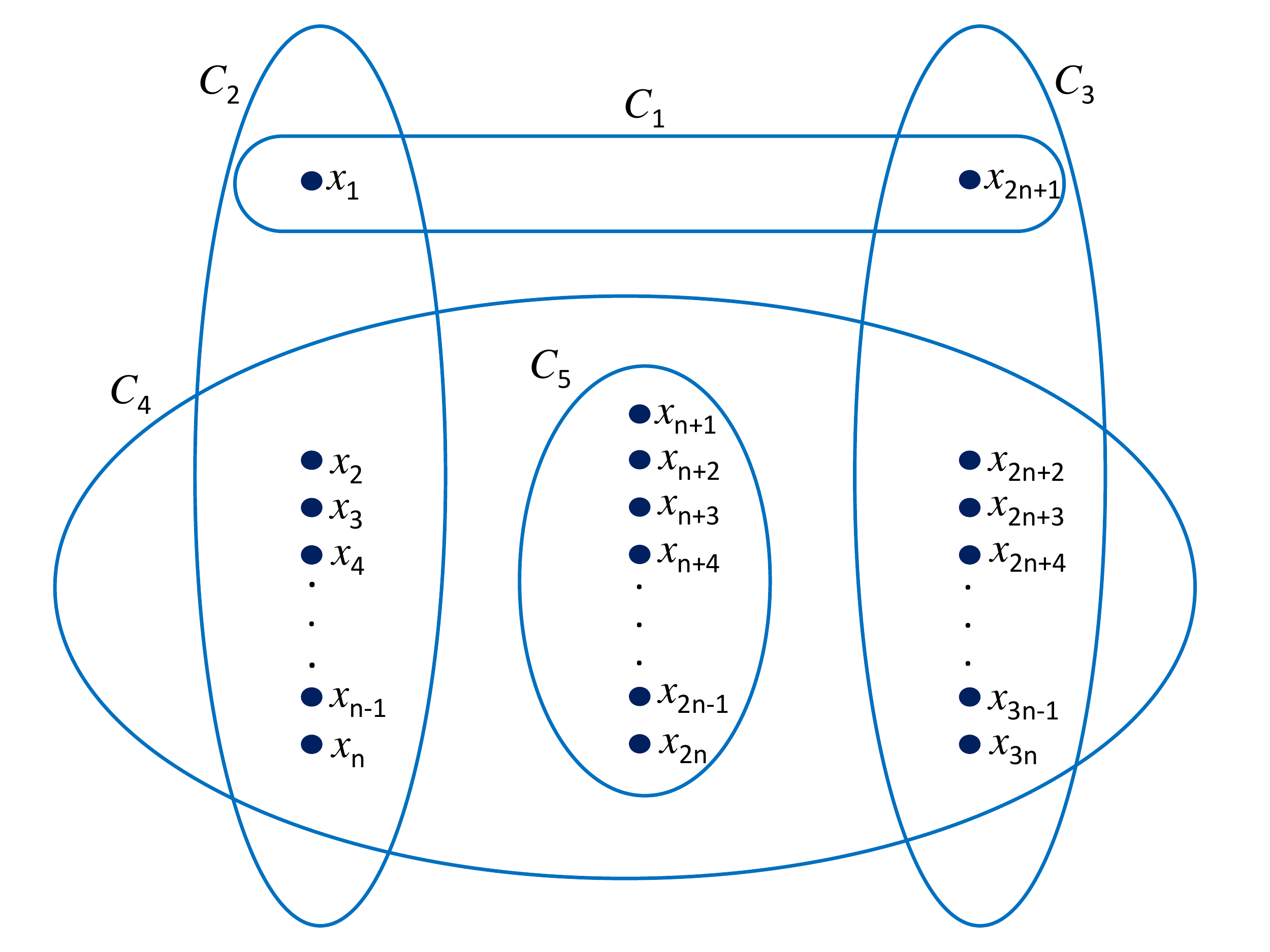}
  \caption{The structure of the constraint problems in Example~\ref{example:Running}}
  \label{fig:RunningExampleG}
\end{figure}

This family of problems is not included in any previously known tractable class,
but will be shown to be tractable using the results of this paper.
\end{example}


As discussed in~\cite{ChenGrohe10-csp-succ-repr}, 
when the constraints in a family of problems have unbounded arity,
the way that the constraints are {\em represented} can significantly affect the
complexity.
Previous work
in this area has assumed that the global constraints have specific representations, 
such as propagators~\cite{Green08:cp-structural},
negative constraints~\cite{gdnf-cohen-repr},
or GDNF/decision diagrams~\cite{ChenGrohe10-csp-succ-repr}, and exploited properties
particular to that representation.  In contrast, here we investigate
the conditions that yield efficiently solvable classes of constraint
problems with global constraints, without requiring any specific
representation. Many global constraints have succinct representations,
so even problems with very simple structures are known to be hard in some cases
\cite{Kutz08:sim-match,Samer11:constraints-tractable}. We will therefore
need to impose some restrictions on the properties of the individual 
global constraints, as well as on the problem structure.

To obtain our results, 
we define a notion of equivalence on assignments and 
a new width measure that identifies variables 
that are constrained in exactly the same way.
We then show that we can replace variables that are equated
under our width measure with a single new variable whose domain
represents the possible equivalence classes of assignments.
Both of these simplification steps, merging variables and equating assignments,
can be seen as techniques for eliminating symmetries in the original problem formulation. 
We describe some sufficient conditions under which these techniques 
provide a polynomial-time reduction to a known tractable case,
and hence identify new tractable classes of constraint problems 
involving global constraints.

\section{Global Constraints and Constraint Problems}

In order to be more precise about the way in which global constraints
are represented, we will extend the standard definition of a
constraint problem.

\begin{definition}[Variables and assignments]
  Let $V$ be a set of variables, each with an associated set of domain
  elements. We denote the set of domain elements (the domain) of a
  variable $v$ by $D(v)$.
  We extend this notation to arbitrary subsets of variables, $W$, 
  by setting $D(W) = \displaystyle\bigcup_{v \in W} D(v)$.

  An {\em assignment} of a set of variables $V$ is a function 
  $\theta : V \rightarrow D(V)$
  that maps every $v \in V$ to an element $\theta(v) \in D(v)$. 
  We denote the restriction of $\theta$ to a set of variables $W \subseteq V$ 
  by $\theta|_W$. We also allow the special assignment
  $\bot$ of the empty set of variables. In particular, for every
  assignment $\theta$, we have $\theta|_\emptyset = \bot$.
\end{definition}

Global constraints have traditionally been defined, somewhat vaguely,
as constraints without a fixed arity, possibly also with a compact
representation of the constraint relation. For example, in
\cite{vanHoeve2006:global-const} a global constraint is defined as ``a
constraint that captures a relation between a non-fixed number of
variables''.

Below, we offer a precise definition similar to the one in
\cite{Bess07:reasoning-global-const}, where the authors define global
constraints for a domain $D$ over a list of variables $\sigma$ as
being given intensionally by a function $D^{|\sigma|} \rightarrow \{0,
1\}$ computable in polynomial time. Our definition differs from this
one in that we separate the general {\em algorithm} of a global
constraint (which we call its {\em type}) from the specific
description.  This separation allows us a better way of measuring the
size of a global constraint, which in turn helps us to establish new
complexity results.

\begin{definition}[Global constraints]
  A \emph{global constraint type} is a parametrised polynomial-time
  algorithm that determines the acceptability of an assignment of a given
  set of variables.

  Each global constraint type, $e$, has an associated set of 
  \emph{descriptions}, $\Delta(e)$. Each description $\delta \in \Delta(e)$
  specifies appropriate parameter values for the algorithm $e$. 
  In particular, each $\delta \in \Delta(e)$ specifies a set of
  variables, denoted by $\vars(\delta)$.

  A \emph{global constraint} $e[\delta]$, where $\delta \in
  \Delta(e)$, is a function that maps assignments of $\vars(\delta)$
  to the set $\{0,1\}$.  Each assignment that is allowed by
  $e[\delta]$ is mapped to 1, and each disallowed assignment is mapped
  to 0.  The \emph{extension} or \emph{constraint relation} of
  $e[\delta]$ is the set of assignments, $\theta$, of $\vars(\delta)$
  such that $e[\delta](\theta) = 1$. We also say that such assignments
  \emph{satisfy} the constraint, while all other assignments
  \emph{falsify} it.
\end{definition}

When we are only interested in describing the set of assignments that
satisfy a constraint, and not in the complexity of determining
membership in this set, we will sometimes abuse notation by writing
$\theta \in e[\delta]$ to mean $e[\delta](\theta) = 1$.

As can be seen from the definition above, a global constraint is not
usually explicitly represented by listing all the assignments that
satisfy it. Instead, it is represented by some description $\delta$
and some algorithm $e$ that allows us to check whether the constraint
relation of $e[\delta]$ includes a given assignment. To stay within
the complexity class \NP, this algorithm is required to run in
polynomial time. As the algorithms for many common global constraints
are built into modern constraint solvers, we measure the {\em size} of
a global constraint's representation by the size of its description.

\begin{example}[EGC]
  \label{example:egc}
  A very general global constraint type is the \emph{extended global
    cardinality} constraint type~\cite{quimper04-gcc-npc,Samer11:constraints-tractable}.
  This form of global constraint is defined by specifying for every
  domain element $a$ a finite set of natural numbers $K(a)$, called
  the cardinality set of $a$. The constraint requires that the number
  of variables which are assigned the value $a$ is in the set $K(a)$,
  for each possible domain element $a$.

  Using our notation, the description $\delta$ of an EGC global
  constraint specifies a function $K_\delta : D(\vars(\delta))
  \rightarrow \mathcal P(\mathbb N)$ that maps each domain element to
  a set of natural numbers.  The algorithm for the EGC constraint then
  maps an assignment $\theta$ to $1$ if and only if, for every domain
  element $a \in D(\vars(\delta))$, we have that $|\{v \in
  \vars(\delta) \mid \theta(v) = a\}| \in K_\delta(a)$.

\end{example}

The cardinality constraint $C_2$ from Example~\ref{example:Running}
can be expressed as an EGC global constraint with description $\delta$
such that $K_\delta(1) = \{1\}$, and $K_\delta(0) = \{n-1\}$.

\begin{example}[Clauses]
  \label{example:clauses}
  We can view the disjunctive clauses used to define propositional satisfiability
  problems as a global constraint type in the following way.
  
  The description $\delta$ of a clause is simply a list of the literals that it contains,
  and $\vars(\delta)$ is the corresponding set of variables.
  The algorithm for the clause then maps any Boolean assignment 
  $\theta$ of $\vars(\delta)$ that
  satisfies the disjunction of the literals specified by $\delta$ to 1, 
  and all other assignments to 0.
  
  Note that a clause forbids precisely one assignment to
  $\vars(\delta)$ (the one that falsifies all of the literals in the
  clause).  Hence the extension of a clause contains
  $2^{|\vars(\delta)|}-1$ assignments, so the size of the constraint
 \emph{relation} grows exponentially with the number of variables, but the
  size of the constraint \emph{description} grows only linearly.
 \end{example}

\begin{example}[Table and negative constraints]
  \label{example:table-const}
  A rather degenerate example of a a global constraint type is 
  the \emph{table} constraint.
  
  In this case the description $\delta$ is simply a list of assignments 
  of some fixed set of variables, $\vars(\delta)$. The algorithm for
  a table constraint then decides, for any
  assignment of $\vars(\delta)$, whether it is included in $\delta$. 
  This can be done in a time which is linear in the size of $\delta$ 
  and so meets the polynomial time requirement. 

  {\em Negative} constraints are complementary to table constraints,
  in that they are described by listing {\em forbidden}
  assignments. The algorithm for a negative constraint $e[\delta]$
  decides, for any assignment of $\vars(\delta)$, whether it
  is {\em not} included in $\delta$. Observe that the clauses
  described in Example~\ref{example:clauses} are a special case of the
  negative constraint type, as they have exactly one forbidden
  assignment.
  
  We observe that any global constraint can be rewritten as a table or
  negative constraint. However, this rewriting will, in general, incur
  an exponential increase in the size of the description.
\end{example}

\begin{definition}[CSP instance]
  An instance of the constraint satisfaction problem (CSP) 
  is a pair $\tup{V, C}$ where $V$ is a finite set of
  \emph{variables}, and $C$ is a set of \emph{global constraints} 
  such that for every $e[\delta] \in C$, $\vars(\delta) \subseteq V$. 
  In a CSP instance, we call $\vars(\delta)$ the \emph{scope} 
  of the constraint $e[\delta]$.

  A \emph{solution} to a CSP instance $\tup{V, C}$ is an assignment $\theta$ of $V$
  which satisfies every global constraint, 
  i.e., for every $e[\delta] \in C$ we have $\theta|_{\vars(\delta)} \in e[\delta]$.
\end{definition}

The general constraint satisfaction problem is clearly NP-complete, so
in the remainder of the paper we shall look for more restricted
versions of the problem that are {\em tractable}, that is, solvable in
polynomial time.

\section{Restricted Classes of Constraint Problems}


First, we are going to consider restrictions on the way that the
constraints in a given instance interact with each other, or, in other
words, the way that the constraint scopes overlap; such restrictions
are known as {\em structural}
restrictions~\cite{Cohen08:unifiedstructural,Gottlob00:acomparison,grohe07-hom-csp-complexity}.

\begin{definition}[Hypergraph]
  A hypergraph $\tup{V, H}$ is a set of vertices $V$ together with a
  set of hyperedges $H \subseteq \mathcal P(V)$.

  Given a CSP instance $P = \tup{V, C}$, the hypergraph of $P$,
  denoted $\hyp(P)$, has vertex set $V$ together with a hyperedge
  $\vars(\delta)$ for every $e[\delta] \in C$.
\end{definition}

One special class of hypergraphs that has received a great deal of
attention is the class of {\em acyclic}
hypergraphs~\cite{Beeri83-acyclic-databases}.  This notion is a
generalisation of the idea of tree-structure in a graph, and has
been very important in the analysis of relational databases.  A
hypergraph is said to be acyclic if repeatedly removing all hyperedges
contained in other hyperedges, and all vertices contained in only a
single hyperedge, eventually deletes all vertices
\cite{Beeri83-acyclic-databases}.

Solving a CSP instance $P$ whose constraints are represented
extensionally (i.e., as table constraints) is known to be tractable if
the hypergraph of $P$, $\hyp(P)$, is
acyclic~\cite{GyssensJeavons94-database-tech}.  Indeed, this has
formed the basis for more general notions of ``bounded
cyclicity''~\cite{GyssensJeavons94-database-tech} or ``bounded
hypertree width''~\cite{Gottlob02:jcss-hypertree}, which have also
been shown to imply tractability for problems with explicitly
represented constraint relations.  However, this is no longer true if
the constraints are global, not even when we have a fixed, finite
domain, as the following examples show.

\begin{example}
  \label{example:singleEGC}
  Any hypergraph containing only a single edge is clearly acyclic (and
  therefore has hypertree width one~\cite{Gottlob02:jcss-hypertree}),
  but the class of CSP instances consisting of a single EGC constraint
  over an unbounded domain is \NP-complete~\cite{quimper04-gcc-npc}.
\end{example}

\begin{example}
  \label{example:3col}
  The \NP-complete problem of 3-colourability
  \cite{Garey79:intractability} is to decide, given a graph $\tup{V,E}$, whether the vertices $V$ can be coloured with three colours
  such that no two adjacent vertices have the same colour.

  We may reduce this problem to a CSP with EGC constraints
  (cf.~Example~\ref{example:egc}) as follows: Let $V$ be the set of
  variables for our CSP instance, each with domain $\{r,g,b\}$. For
  every edge $\tup{v, w} \in E$, we post an EGC constraint with scope
  $\{v, w\}$, parametrised by the function $K$ such that $K(r) = K(g)
  = K(b) = \{0,1\}$. Finally, we make the hypergraph of this CSP
  instance acyclic by adding an EGC constraint with scope $V$
  parametrised by the function $K'$ such that $K'(r) = K'(g) = K'(b) =
  \{0,\ldots,|V|\}$. This reduction clearly takes polynomial time,
  and the hypergraph of the resulting instance is acyclic.


\end{example}

These examples indicate that when dealing with implicitly represented
constraints we cannot hope for tractability using structural
restrictions alone.  We are therefore led to consider {\em hybrid}
restrictions, which restrict both the nature of the constraints and the
structure at the same time. 

\begin{definition}[Constraint \languageword]
  A \emph{constraint \languageword} is a set of global constraints. A
  CSP instance $\tup{V, C}$ is said to be over a constraint
  \languageword\ $\languagesymbol$ if for every $e[\delta] \in C$ we
  have $e[\delta] \in \languagesymbol$.
\end{definition}

Previous work on the complexity of constraint problems has restricted
the {\em extensions} of the constraints to a specified set of {\em
  relations}, known as a constraint {\em
  language}~\cite{bulatov05:classifying-constraints}.  This is an
appropriate form of restriction when all constraints are given
explicitly, as table constraints.  However, here we work with global
constraints where the relations are often implicit, and this can
significantly alter the complexity of the corresponding problem
classes, as we will illustrate below.  Hence we allow a more general
form of restriction on the constraints by specifying a constraint
\languageword\ containing all allowed constraints.


\begin{definition}[Restricted CSP class]
\label{def:CSPrestricted}
Let $\languagesymbol$ be a constraint \languageword, and let $\mathcal H$ be a class of
hypergraphs.  We define $\CSP(\mathcal H,\languagesymbol)$ to be the class of
CSP instances over $\languagesymbol$ whose hypergraphs are in $\mathcal H$.

\end{definition}

Using Definition~\ref{def:CSPrestricted}, we will restate an earlier
structural tractability result, which will form the basis for our
results in Section~\ref{sect:main-results}.  

\begin{definition}[Treewidth]
  A \emph{tree decomposition} of a hypergraph $\tup{V, H}$ is a pair
  $\langle T, \lambda \rangle$ where $T$ is a tree and $\lambda$ is a
  labelling function from nodes of $T$ to subsets of $V$, such that
  \begin{enumerate}
  \item for every $v \in V$, there exists a node $t$ of $T$ such that
    $v \in \lambda(t)$,
  \item for every hyperedge $h \in E$, there exists a node $t$ of $T$
    such that $h \subseteq \lambda(t)$, and
  \item for every $v \in V$, the set of nodes $\{ t \mid v \in
    \lambda(t) \}$ induces a connected subtree of $T$.
  \end{enumerate}

  The \emph{width} of a tree decomposition is $\max(\{ |\lambda(t)|-1
  \mid t \mbox{ node of } T\})$. The \emph{treewidth} $\tw(G)$ of a
  hypergraph $G$ is the minimum width over all its tree
  decompositions.

  Let $\mathcal H$ be a class of hypergraphs, and define $\tw(\mathcal
  H)$ to be the maximum treewidth over the hypergraphs in $\mathcal
  H$. If $\tw(\mathcal H)$ is unbounded we write $\tw(\mathcal H) =
  \infty$; otherwise $\tw(\mathcal H) < \infty$.
\end{definition}





We can now restate using the language of global constraints the
following result, from Dalmau et
al.~\cite{Dalmau02:csp-tractability-cores}, which builds on several
earlier results~\cite{DechterPearl89:treeclustering,Freuder90:ktree}.

\begin{theorem}[\cite{Dalmau02:csp-tractability-cores}]
  \label{thm:dalmau}
  Let $\languagesymbol$ be a constraint \languageword\ and $\mathcal H$ a 
  class of hypergraphs. $\CSP(\mathcal H, \languagesymbol)$ is
  tractable if $\tw(\mathcal H) < \infty$.
\end{theorem}


Observe that the family of constraint problems described in
Example~\ref{example:Running} is not covered by the above result,
because the treewidth of the associated hypergraphs is unbounded.



\section{Cooperating Constraint \Languageword{}s}
\label{sect:language-properties}

Whenever constraint scopes overlap, we may ask whether the possible
assignments to the variables in the overlap are essentially
different. It may be that some assignments extend to precisely the
same satisfying assignments in each of the overlapping constraints. If
so, we may as well identify such assignments.

\begin{definition}[Disjoint union of assignments]
 \label{def:disjoint-union}
  Let $\theta_1$ and $\theta_2$ be two assignments of disjoint sets of
  variables $V_1$ and $V_2$, respectively. The disjoint union of
  $\theta_1$ and $\theta_2$, denoted $\theta_1 \oplus \theta_2$, is the 
  assignment of $V_1 \cup V_2$ such that $(\theta_1 \oplus \theta_2)(v) =
  \theta_1(v)$ for all $v \in V_1$, and $(\theta_1 \oplus \theta_2)(v) =
  \theta_2(v)$ for all $v \in V_2$.
%
\end{definition}

\begin{definition}[Projection]
   Let $\Theta$ be a set of assignments of a set of variables $V$. The
   \emph{projection} of $\Theta$ onto a set of variables $X \subseteq V$ is
   the set of assignments $\pi_X(\Theta) = \{\theta|_X \mid \theta \in
   \Theta\}$. 
\end{definition}
   
Note that when $\Theta = \emptyset$ we have $\pi_X(\Theta) = \emptyset$ for any set $X$, 
but when $X = \emptyset$ and $\Theta \neq \emptyset$, we have $\pi_X(\Theta) = \{\bot\}$.

\begin{definition}[Assignment extension]
  Let $e[\delta]$ be a global constraint, 
  and $X \subseteq \vars(\delta)$. For every assignment $\mu$ of $X$, let
  $\ext(\mu, e[\delta]) = \pi_{\vars(\delta) - X}(\{\theta \in e[\delta] \mid
  \theta|_X = \mu\})$.
\end{definition}

In other words, for any assignment $\mu$ of $X$, the set $\ext(\mu, e[\delta])$ is the
set of assignments of $\vars(\delta) - X$ that extend $\mu$ to a
satisfying assignment for $e[\delta]$; 
i.e.,~those assignments $\theta$ for which $\mu \oplus \theta \in e[\delta]$.

\begin{definition}[Extension equivalence]
  \label{def:equiv-tuples}
  Let $e[\delta]$ be a global constraint,
  and $X \subseteq \vars(\delta)$.
  We say that two assignments $\theta_1, \theta_2$ to
  $X$ are \emph{extension equivalent} on $X$ with respect to $e[\delta]$ if
  $\ext(\theta_1, e[\delta]) = \ext(\theta_2,e[\delta])$. 
  We denote this equivalence relation by $\eqt[e[\delta], X]$; 
  that is, $\eqt[e[\delta], X](\theta_1, \theta_2)$
  holds if and only if $\theta_1$ and $\theta_2$ are extension equivalent
  on $X$ with respect to $e[\delta]$. 
\end{definition}

In other words, two assignments to some subset of the variables of a
constraint $e[\delta]$ are extension equivalent if every assignment to
the rest of the variables combines with both of them to give either
two assignments that satisfy $e[\delta]$, or two that falsify
it. 

\begin{example}
\label{example:clauses2classes}
Consider the special case of extension equivalence with respect to a
clause (cf. Example~\ref{example:clauses}).

Given any clause $e[\delta]$, and any non-empty set of variables $X
\subseteq \vars(\delta)$, any assignment to $X$ will either satisfy
one of the corresponding literals specified by $\delta$, or else
falsify all of them.  If it satisfies at least one of them, then any
extension will satisfy the clause, so all such assignments are
extension equivalent. If it falsifies all of them, then an extension
will satisfy the clauses if and only if it satifies one of the other
literals.  Hence the equivalence relation $\eqt[e[\delta], X]$ has
precisely 2 equivalence classes, one containing the single assignment
that falsifies all the literals corresponding to X, and one containing
all other assignments.
\end{example}



\begin{definition}[Intersecting variables]
  Let $S$ be a set of global constraints. We write $\iv(S)$ for the
  set of variables common to all of their scopes, that is, $\iv(S) =
  \displaystyle\bigcap_{e[\delta] \in S} \vars(\delta)$.
\end{definition}

\begin{definition}[Join]
For any set $S$ of global constraints, we define 
the {\em join} of $S$, denoted $\join(S)$, to be a global constraint
$e'[\delta']$ with 
$\vars(\delta') = \displaystyle\bigcup_{e[\delta] \in S} \vars(\delta)$ 
such that for any assignment $\theta$ to $\vars(\delta')$, we have 
$\theta \in e'[\delta']$ if and only if for every $e[\delta] \in S$ we have
   $\theta|_{\vars(\delta)} \in e[\delta]$.
\end{definition}

The join of a set of global constraints may have no simple compact 
description, and computing its extension may be computationally expensive. 
However, we introduce this construct simply in order to describe the combined effect
of a set of global constraints in terms of a single constraint.

\begin{example}
  Let $V = \{v_1,\ldots,v_n\}$, for some $n \geq 3$, be a set of variables 
  with $D(v_i) = \{a, b, c\}$, and let $S = \{e_1[\delta_1], e_2[\delta_2]\}$ 
  be a set of two global constraints as defined below:
  \begin{itemize}
    \item 
    $e_1[\delta_1]$ is a table constraint with $\vars(\delta_1) =
    \{v_1,\ldots,v_{n-1}\}$ which enforces {\em equality}, 
    i.e., $\delta_1 = \{\theta_a,\theta_b, \theta_c\}$, 
    where for each $x \in D(V)$ and $v \in \vars(\delta_1)$, $\theta_x(v) = x$.
    
    \item 
    $e_2[\delta_2]$ is a negative constraint with $\vars(\delta_2) =
    \{v_2,\ldots,v_n\}$ which enforces a {\em not-all-equal} condition, 
    i.e., $\delta_2 = \{\theta_a, \theta_b, \theta_c\}$, 
    where for each $x \in D(V)$ and $v \in \vars(\delta_2)$, $\theta_x(v) = x$.
  \end{itemize}

  We will use substitution notation to write assignments explicitly; 
  thus, an assignment of $\{v, w\}$ that assigns
  $a$ to both variables is written $\{v/a, w/a\}$.

  We have that $\iv(S) = \{v_2,\ldots,v_{n-1}\}$. The equivalence classes of
  assignments to $\iv(S)$ under $\eqt[\join(S), \iv(S)]$ are
  $\{\{v_2/a,\ldots,v_{n-1}/a\}\}$, 
  $\{\{v_2/b,\ldots,v_{n-1}/b\}\}$, and 
  $\{\{v_2/c,\ldots,v_{n-1}/c\}\}$,
  each containing the single assignment shown, as well as (for $n > 3$)
  a final class containing all other assignments, 
  for which we can choose an arbitrary representative assignment, $\theta_0$,
  such as $\{v_2/a,v_3/b,\ldots,v_{n-1}/b\}$.
  
  Each assignment in the first 3 classes has just 2 possible extensions
  that satisfy $\join(S)$,
  since the value assigned to $v_1$ must equal the value assigned to 
  $v_2,\ldots,v_{n-1}$, and the value assigned to $v_n$ must be different.
  The assignment $\theta_0$ has no extensions, 
  since $\ext(\theta_0,e_1[\delta_1]) = \emptyset$.
  
  Hence the number of equivalence classes in $\eqt[\join(S),\iv(S)]$ is
  at most 4, even though the total number of 
  possible assignments of $\iv(S)$ is $3^{n-2}$
\end{example}

\begin{definition}[Cooperating constraint \languageword]
  \label{def:cooperating-language}
  We say that a constraint \languageword\ $\languagesymbol$ is a \emph{cooperating}
  \languageword\ if for any finite set of global constraints $S \subseteq
  \languagesymbol$,
  we can compute a set of assignments of the variables $\iv(S)$ 
  containing at least one representative 
  of each equivalence class of $\eqt[\join(S), \iv(S)]$ in polynomial time
  in the size of $\iv(S)$ and the total size of the constraints in $S$.
\end{definition}

Note that this definition requires two things.  First, that the number
of equivalence classes in the equivalence relation $\eqt[\join(S),
\iv(S)]$ is bounded by some fixed polynomial in the size of $\iv(S)$
and the size of the constraints in $S$.  Secondly, that a suitable set
of representatives for these equivalence classes can be computed
efficiently from the constraints.




\begin{example}
\label{example:clausescooperate}
Consider a constraint \languageword\ consisting entirely of clauses
(of arbitrary arity).  It was shown in
Example~\ref{example:clauses2classes} that for any clause $e[\delta]$
and any non-empty $X \subseteq \vars(\delta)$ the equivalence relation
$\eqt[e[\delta], X]$ has precisely 2 equivalence classes.

If we consider some finite set, $S$, of clauses, then a similar
argument shows that the equivalence relation $\eqt[\join(S), \iv(S)]$
has at most $|S|+1$ classes.  These are given by the single
assignments of the variables in $\iv(S)$ that falsify the literals
corresponding to the variables of $\iv(S)$ in each clause (there are
at most $|S|$ of these --- they may not all be distinct) together with
at most one further equivalence class containing all other assignments
(which must satisfy at least one literal in each clause of $S$).

Hence the total number of equivalence classes in the equivalence
relation $\eqt[\join(S), \iv(S)]$ increases at most linearly with the
number of clauses in $S$, and a representative for each class can be
easily obtained from the descriptions of these clauses, by projecting
the falsifying assignments down to the set of common variables,
$\iv(S)$, and adding at most one more, arbitrary, assignment.

By same argument, if we consider some finite set, $S$, of table
constraints, then the equivalence relation $\eqt[\join(S), \iv(S)]$
has at most one class for each assignment allowed by each table
constraint in $S$, together with at most one further class containing
all other assignments.
\end{example}

In general, arbitrary EGC constraints (cf.~Example~\ref{example:egc})
do not form a cooperating \languageword.  However, we will show that
if we bound the size of the variable domains, then the resulting EGC
constraints do form a cooperating \languageword.

\begin{definition}[Counting function]
  Let $X$ be a set of variables with domain $D = \bigcup_{x \in X} D(x)$.  
  A {\em counting function} for $X$ is any function
  $K : D \rightarrow \mathbb N$ such that $\sum_{a \in D}
  K(a) = |X|$.

  Every assignment $\theta$ to $X$ defines a corresponding counting
  function $K_\theta$ given by $K_\theta(a) = |\{x \in X \mid
  \theta(x) = a\}|$ for every $a \in D$.
\end{definition}

It is easy to verify that no EGC constraint can distinguish two
assignments with the same counting function; for any EGC constraint,
either both assignments satisfy it, or they both falsify it. It
follows that two assignments with the same counting function are
extension equivalent with respect to EGC constraints. 

\begin{definition}[Counting constraints]
  \label{def:counting-constraint}
  A global constraint $e[\delta]$ is called a \emph{counting constraint} 
  if, for any two assignments $\theta_1,\theta_2$ of $\vars(\delta)$ 
  which have the same counting function, either
  $\theta_1,\theta_2 \in e[\delta]$ or $\theta_1,\theta_2 \not\in e[\delta]$.
\end{definition}

EGC constraints are not the only constraint type with this property.
Constraints that require the sum (or the product) of the values of all
variables in their scope to take a particular value, and constraints
that require the minimum (or maximum) value of the variables in their
scope to take a certain value, are also counting constraints.

Another example is given by the NValue constraint type, 
which requires that the number of distinct domain values 
taken by an assignment is a member of a specified set of acceptable numbers.

\begin{example}[NValue constraint type \cite{Beldiceanu:2001,bess-nvalue}]
  In an NValue constraint, $e[\delta]$, the description $\delta$
  specifies a finite set of natural numbers $L_\delta \subset \mathbb
  N$. The algorithm $e$ maps an assignment $\theta$ to 1 if
  $|\{\theta(v)\mid v \in \vars(\delta)\}| \in L_\delta$.
\end{example}


The reason for introducing counting functions is the following key
property, previously noted by Bulatov and Marx
\cite{Bulatov10:lmcs-complexity}.

\begin{property}
  \label{prop:num-counting-funcs}
  The number of possible counting functions for a set of variables $X$
  is at most $\binom{|X|+|D|-1}{|D|-1} = O(|X|^{|D|})$, where $D =
  \bigcup_{x \in X} D(x)$.
\end{property}
\begin{proof}
  If every variable $x \in X$ has $D$ as its set of domain elements,
  that is, $D(x) = D$, then every counting function corresponds to a
  distinct way of partitioning $|X|$ variables into at most $|D|$
  boxes. There are $\binom{|X|+|D|-1}{|D|-1}$ ways of doing so
  \cite[Section~2.3.3]{handbook-discrete}. On the other hand, if there
  are variables $x \in X$ such that $D(x) \subset D$, then that
  disallows some counting functions.
\end{proof}




\begin{theorem}
  \label{thm:coop-language}
  Any constraint \languageword\ that contains only 
  counting constraints with bounded domain size, 
  table constraints, and negative constraints,  
  is a cooperating \languageword.
\end{theorem}
\begin{proof}
  Let $\languagesymbol$ be a constraint \languageword\ containing only 
  global constraints of the specified types, 
  and let $S \subseteq \languagesymbol$ be a finite subset of $\languagesymbol$. 
  Partition $S$ into two subsets: $S^C$, containing only counting constraints and
  $S^{\pm}$ containing only table and negative constraints.


  Let $\mathcal K$ be a set containing assignments of $\iv(S)$, such
  that for every counting function $K$ for $\iv(S)$, there is some
  assignment $\theta_K \in \mathcal K$ with $K_\theta = K$.  By
  Property~\ref{prop:num-counting-funcs}, the number of counting
  functions for $\iv(S)$ is bounded by $O(|\iv(S)|^d)$, where $d$ is
  the bound on the domain size for the counting constraints in
  $\languagesymbol$.  Hence such a set $\mathcal K$ can be computed in
  polynomial time in the size of $\iv(S)$.
  
  For each constraint in $S^{\pm}$ we have that the description is a
  list of assignments (these are the allowed assignments for the table
  constraints and the forbidden assignments for the negative
  constraints, see Example~\ref{example:table-const}).

  As we described in Example~\ref{example:clausescooperate}, for each
  table constraint $e[\delta] \in S$, we can obtain a representative
  for each equivalence class of $\eqt[e[\delta], \iv(S)]$ by taking
  the projection onto $\iv(S)$ of each allowed assignment, which we
  can denote by $\pi_{\iv(S)}(\delta)$, together with at most one
  further, arbitrary, assignment, $\theta_0$, that is not in this set.
  This set of assignments contains at least one representative for
  each equivalence class of $\eqt[e[\delta], \iv(S)]$ (and possibly
  more than one representative for some of these classes).
    
  Similarly, for each negative constraint $e[\delta] \in S$, we can
  obtain a representative for each equivalence class of
  $\eqt[e[\delta], \iv(S)]$, by taking the projection onto $\iv(S)$ of
  each forbidden assignment, which we can again denote by
  $\pi_{\iv(S)}(\delta)$, together with at most one further,
  arbitrary, assignment, $\theta_0$, that is not in this set.  
  
  Now consider the set of assignments $\mathcal A = \mathcal K \cup
  \{\theta_0\} \cup \displaystyle\bigcup_{e[\delta] \in S^{\pm}}
  \pi_{\iv(S)}(\delta)$, where $\theta_0$ is an arbitrary assignment
  of $\iv(S)$ which does not occur in $\pi_{\iv(S)}(\delta)$ for any
  $e[\delta] \in S$ (if such an assignment exists).  We claim that
  this set of assignments contains at least one representative for
  each equivalence class of $\eqt[\join(S), \iv(S)]$ (and possibly
  more than one for some classes).
  
  To establish this claim we will show that any assignment $\theta$ of
  $\iv(S)$ that is not in $\mathcal A$ must be extension equivalent to
  some member of $\mathcal A$.  Let $\theta$ be an assignment of
  $\iv(S)$ that is not in $\mathcal A$ (if such an assignment exists).
  If $S^{\pm}$ contains any positive constraints, then $\theta$ has an
  empty set of extensions to these constraints, and hence is extension
  equivalent to $\theta_0$.  Otherwise, any extension of $\theta$ will
  satisfy all negative constraints in $S^{\pm}$, so the extensions of
  $\theta$ that satisfy $\join(S)$ are completely determined by the
  counting function $K_\theta$. In this case $\theta$ will be
  extension equivalent to some element of $\mathcal K$.
    
  Moreover, the set of assignments $\mathcal A$ can be computed from $S$
  in polynomial time in the the size of $\iv(S)$ and the total size of the descriptions
  of the constraints in $S^{\pm}$.
  Therefore, $\languagesymbol$ is a cooperating \languageword\ as described in 
  Definition~\ref{def:cooperating-language}.
\end{proof}

\begin{example}
\label{example:Runningcooperates}
By Theorem~\ref{thm:coop-language}, the constraints in
Example~\ref{example:Running} form a cooperating \languageword. 
\end{example}

\section{Polynomial-time Reductions}
\label{sect:main-results}

In this section, we will show that, for any constraint problem over a
cooperating \languageword, a set of variables that all occur in
exactly the same set of constraint scopes can be replaced by a single
new variable with an appropriate domain, to give a polynomial-time
reduction to a smaller problem.


\begin{definition}[Dual of a hypergraph]
  Let $G = \tup{V, H}$ be a hypergraph. The \emph{dual} $G^*$ of $G$
  is a hypergraph with vertex set $H$ and a hyperedge $\{ h \in H \mid
  v \in h \}$ for every $v \in V$. For a class $\mathcal H$ of
  hypergraphs, let $\mathcal H^* = \{G^* \mid G \in \mathcal H\}$.
\end{definition}

\begin{example}
  \label{example:dual}
  Consider the hypergraph $G$ in Figure~\ref{fig:RunningExampleG}. The
  dual, $G^*$, of this hypergraph has vertex set $\{C_1, C_2, C_3,
  C_4, C_5\}$ and five hyperedges $\{C_1,C_2\}$, $\{C_1,C_3\}$,
  $\{C_2,C_4\}$, $\{C_3,C_4\}$ and $\{C_4,C_5\}$. This transformation
  is illustrated in Figure~\ref{fig:RunningExampleGtoGD}.

\begin{figure}[ht]
  \centering
  \includegraphics[width=0.8\textwidth]{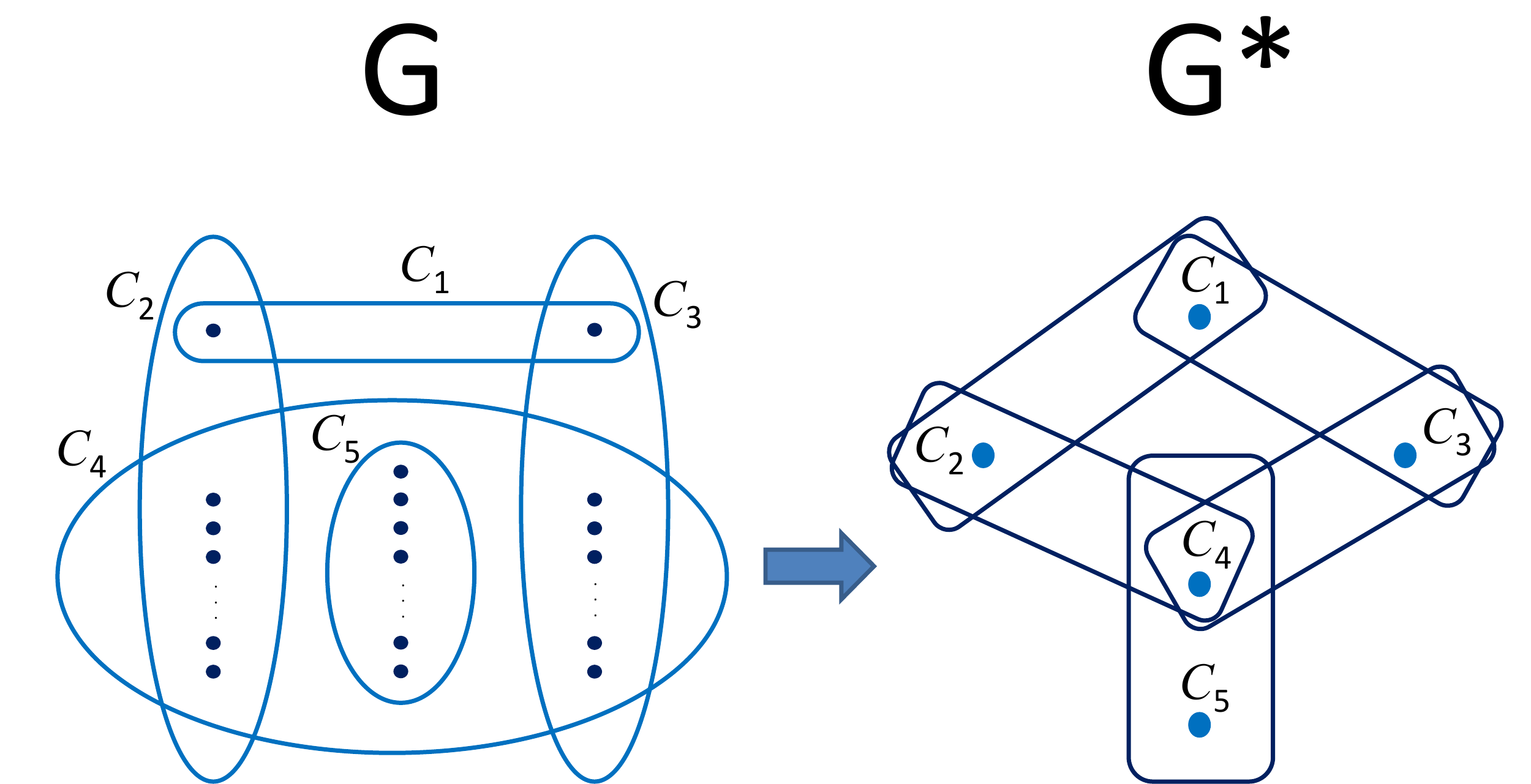}
  \caption{$G$ and $G^{*}$ from Example~\ref{example:dual}}
  \label{fig:RunningExampleGtoGD}
\end{figure}
\end{example}

Note that the dual of the dual of a hypergraph is not necessarily the
original hypergraph, since we do not allow multiple identical
hyperedges.

\begin{example}
  \label{example:dualdual}
  Consider the dual hypergraph $G^*$ defined in
  Example~\ref{example:dual}. Taking the dual of this hypergraph
  yields $G^{**}$, with vertex set $\{h_1,\ldots,h_5\}$ (corresponding
  to the 5 hyperedges in $G^*$) and 5 distinct hyperedges, as shown in
  Figure~\ref{fig:RunningExampleGDtoGDD}.

\begin{figure}[ht]
  \centering
  \includegraphics[width=0.8\textwidth]{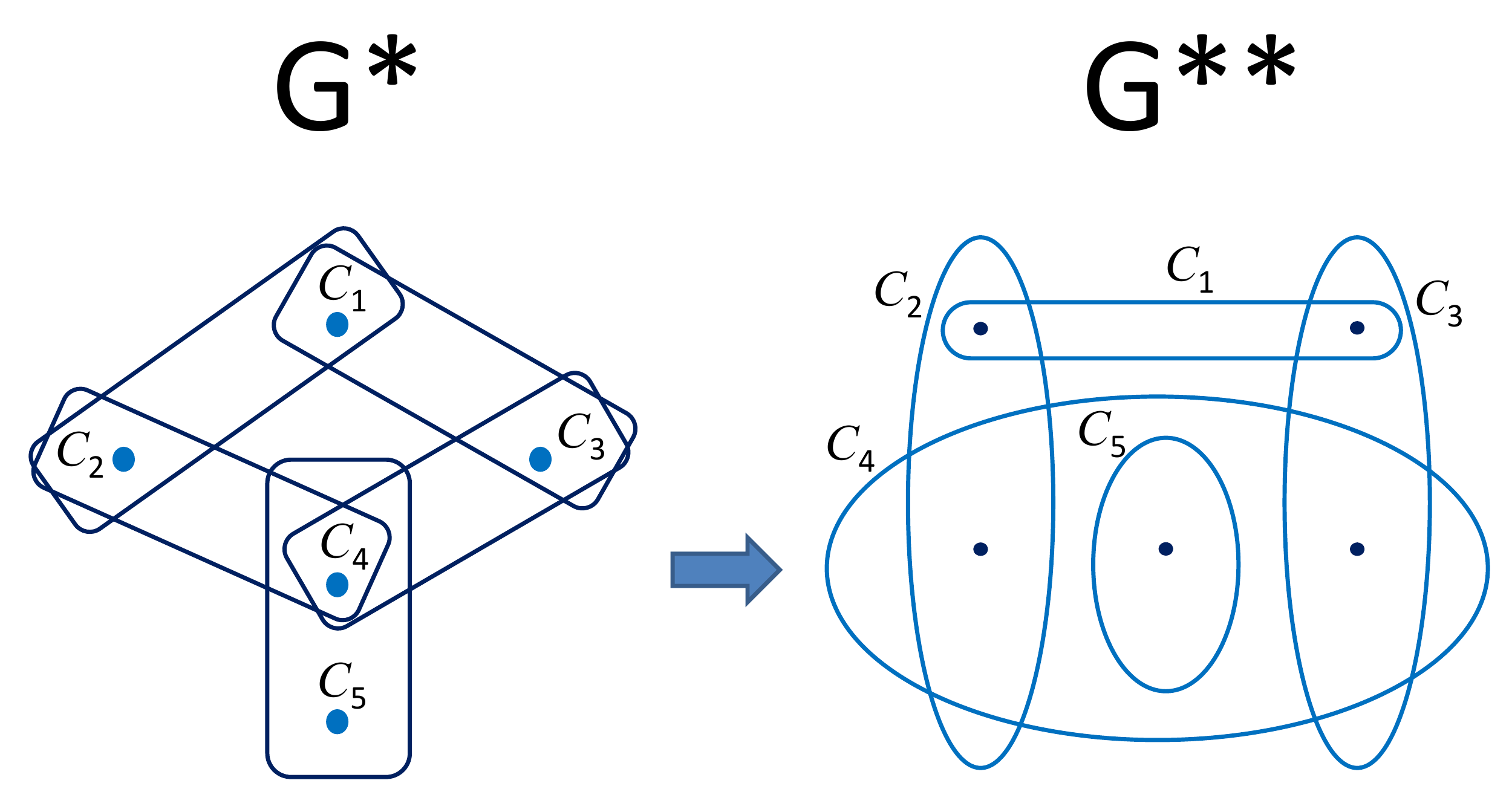}
  \caption{$G^*$ and $G^{**}$ from Example~\ref{example:dualdual}}
  \label{fig:RunningExampleGDtoGDD}
\end{figure}
\end{example}

In the example above, taking the dual of a hypergraph twice had the
effect of merging precisely those sets of variables that occur in the
same set of hyperedges.  It is easy to verify that this is true in
general: Taking the dual twice equates precisely those variables that
occur in the same set of hyperedges.

\begin{lemma}
  \label{lem:dual-vars}
  For any hypergraph $G$, the hypergraph $G^{**}$ has precisely one
  vertex corresponding to each maximal subset of vertices of $G$ that
  occur in the same set of hyperedges.
\end{lemma}


Next, we combine the idea of the dual with the usual notion of
treewidth to create a new measure of width.

\begin{definition}[twDD]
  \label{def:tdd}
  Let $G$ be a hypergraph. The treewidth of the dual of the dual (twDD)
  of $G$ is $\TDD(G) = \tw(G^{**})$.

  For a class of hypergraphs $\mathcal H$, we define $\TDD(\mathcal H)
  = \tw(\mathcal H^{**})$.
\end{definition}

\begin{example}
\label{example:RunningtwDD}
Consider the class $\mathcal H$ of hypergraphs of the family of
problems described in Example~\ref{example:Running}.  Whatever the
value of $n$, the dual hypergraph, $G^*$ is the same, as shown in
Figure~\ref{fig:RunningExampleGtoGD}.  Hence for all problems in this
family the hypergraph $G^{**}$ is as shown in
Figure~\ref{fig:RunningExampleGDtoGDD}, and can be shown to have
treewidth 3.  Hence $\TDD(\mathcal H) = 3$.
\end{example}

When replacing a set of variables in a CSP instance with a single
variable, we will use the following definition.

\begin{definition}[Quotient of a CSP instance]
  \label{def:equiv-repr}
  Let $P = \tup{V,C}$ be a CSP instance and $X \subseteq V$ be a
  non-empty subset of variables that all occur in the scopes of the
  same set $S$ of constraints. The {\em quotient} of $P$ with respect
  to $X$, denoted $P^X$, is defined as follows.
  \begin{itemize}
  \item
  The variables of $P^X$ are given by $V^X = (V - X) \cup \{v_X\}$, 
  where $v_X$ is a fresh variable, and 
  the domain of $v_X$ is the set of equivalence classes of $\eqt[\join(S),X]$.
  
  \item
  The constraints of $P^X$ are unchanged, except that each constraint $e[\delta] \in S$
  is replaced by a new constraint $e^X[\delta^X]$, where
  $\vars(\delta^X) = (\vars(\delta) - X) \cup \{v_X\}$.
  For any assignment $\theta$ of $\vars(\delta^X)$,
  we define $e^X[\delta^X](\theta)$ to be 1 if and only if 
  $\theta|_{\vars(\delta) - X} \oplus \mu \in e[\delta]$, where $\mu$ is 
  a representative of the equivalence class $\theta(v_X)$.
  \end{itemize}
\end{definition}

We note that, by Definition~\ref{def:equiv-tuples}, the value of
$e^X[\delta^X]$ specified in Definition~\ref{def:equiv-repr} is
well-defined, that is, it does not depend on the specific
representative chosen for the equivalence class $\theta(v_X)$, since
each representative has the same set of possible extensions.

\begin{lemma}
  \label{lem:equiv-repr}
  Let $P = \tup{V,C}$ be a CSP instance and $X \subseteq V$ be a
  non-empty subset of variables that all occur in the scopes of the
  same set of constraints. The instance $P^X$ has a solution if and
  only if $P$ has a solution.
\end{lemma}
\begin{proof}[Sketch]
  Let $P = \tup{V, C}$ and $X$ be given, and let $S \subseteq C$ be
  the set of constraints $e[\delta]$ such that $X \subseteq \vars(\delta)$.
  
  Construct the instance $P^X$ as specified in
  Definition~\ref{def:equiv-repr}. Any solution to $P$ can be
  converted into a corresponding solution for $P^X$, and vice versa.
  This conversion process just involves replacing the part of the
  solution assignment that gives values to the variables in the set
  $X$ with an assignment that gives a suitable value to the new
  variable $v_X$.


\end{proof}


\begin{theorem}
  \label{thm:dd-csp}
  Any CSP instance $P$ can be converted to an instance $P'$  
  with $\hyp(P') = \hyp(P)^{**}$, 
  such that $P'$ has a solution if and only if $P$ does.
  Moreover, if $P$ is over a cooperating \languageword, this conversion can be 
  done in polynomial time.
\end{theorem}
\begin{proof}
  Let $P = \tup{V,C}$ be a CSP instance.  For each variable $v \in V$
  we define $S(v) = \{e[\delta] \in C \mid v \in \vars(\delta)\}$ We
  then partition the vertices of $P$ into subsets $X_1,\ldots,X_k$,
  where each $X_i$ is a maximal subset of variables $v$ that share the
  same value for $S(v)$.
  
  We initially set $P_0 = P$. Then, for each $X_i$ in turn, we set
  $P_i = (P_{i-1})^X$.  Finally we set $P' = P_k$. By
  Lemma~\ref{lem:dual-vars}, $\hyp(P') = \hyp(P)^{**}$, and by
  Lemma~\ref{lem:equiv-repr}, $P'$ has a solution if and only if $P$
  has a solution.

  Finally, if $P$ is over a cooperating \languageword, then by
  Definition~\ref{def:cooperating-language}, we can compute the
  domains of each new variable introduced in polynomial time in the
  size of each $X_i$ and the total size of the constraints.  Hence we
  can compute $P'$ in polynomial time.
\end{proof}

Using Theorem~\ref{thm:dd-csp}, 
we can immediately get a new tractable CSP class by extending Theorem~\ref{thm:dalmau}.

\begin{theorem}
  \label{thm:tractable-hyp-class}
  Let $\languagesymbol$ be a constraint \languageword\ and $\mathcal
  H$ a class of hypergraphs.  $\CSP(\mathcal H, \languagesymbol)$ is
  tractable if $\languagesymbol$ is a cooperating \languageword{} and
  $\TDD(\mathcal H) < \infty$.
\end{theorem}
\begin{proof}
  Let $\languagesymbol$ be a cooperating \languageword, $\mathcal H$ a class of
  hypergraphs such that $\TDD(\mathcal H) < \infty$, and $P \in
  \CSP(\mathcal H, \languagesymbol)$. 
  Reduce $P$ to a CSP instance $P'$ using Theorem~\ref{thm:dd-csp}.
  By Definition~\ref{def:tdd}, since $\hyp(P') = \hyp(P)^{**}$,
  $\tw(\hyp(P')) < \infty$, which means that $P'$ satisfies the
  conditions of Theorem~\ref{thm:dalmau}, and hence can be solved in
  polynomial time.
\end{proof}

Recall the family of constraint problems described in
Example~\ref{example:Running} at the start of this paper.  Since the
constraints in this problem form a cooperating \languageword\ 
(Example~\ref{example:Runningcooperates}), and all instances have
bounded twDD (Example~\ref{example:RunningtwDD}), this family of
problems is tractable by Theorem~\ref{thm:tractable-hyp-class}.

\section{Summary and Future Work}

We have identified a novel tractable class of constraint problems with
global constraints.  In fact, our results generalize several
previously studied classes of
problems~\cite{Dalmau02:csp-tractability-cores}.  Moreover, this is
the first representation-independent tractability result for
constraint problems with global constraints.

Our new class is defined by restricting both the nature of the
constraints and the way that they interact. As demonstrated in
Example~\ref{example:singleEGC}, instances with a single global
constraint may already be \NP-complete \cite{quimper04-gcc-npc}, so we
cannot hope to achieve tractability by structural restrictions alone.
In other words, notions such as bounded degree of
cyclicity~\cite{GyssensJeavons94-database-tech} or bounded hypertree
width~\cite{Gottlob02:jcss-hypertree} are not sufficient to ensure
tractability in the framework of global constraints, where the arity
of individual constraints is unbounded.  This led us to introduce the
notion of a cooperating constraint \languageword, which is
sufficiently restricted to ensure that an individual constraint is
always tractable.

However, this restriction on the nature of the constraints is still not 
enough to ensure tractability on any structure:
Example~\ref{example:3col} demonstrates 
that not all structures are tractable even with a cooperating constraint \languageword.
In fact, a family of problems with acyclic structure (hypertree width one)
over a cooperating constraint \languageword\ can still be NP-complete.
This led us to investigate restrictions on the structure that are sufficient
to ensure tractability for all instances over a cooperating \languageword.
In particular, we have shown that it is sufficient to ensure that 
the dual of the dual of the hypergraph of the instance has bounded treewidth.

An intriguing open question is whether there are other restrictions
on the nature of the constraints or the structure of the instances that 
are sufficient to ensure tractability in the framework of global
constraints. Very little work has been done on this question, apart from the 
pioneering work of Bulatov and Marx~\cite{Bulatov10:lmcs-complexity}, 
which considered only a single global cardinality constraint, along with arbitrary 
table constraints, and of Chen and Dalmau~\cite{ChenGrohe10-csp-succ-repr}
on two specific succinct representations. 
Almost all other previous work on tractable classes has
considered only table constraints. This may be one reason why such work has
had little practical impact on  the design of constraint solvers, which rely
heavily on the use of in-built special-purpose global constraints.

We see this paper as a first step in the development of a more
robust and applicable theory of tractability for global constraints.



\end{document}